%% file: suppid_arxiv.tex
\documentclass{article}

\usepackage[utf8]{inputenc} 
\usepackage[round]{natbib}
\usepackage[english]{babel}
\usepackage{hyperref}
\usepackage{amsmath,amsthm,amssymb}
\usepackage{algorithm}
\usepackage{algorithmic}
\usepackage[group-separator={,}]{siunitx}
\usepackage[titlenumbered,ruled,noend,algo2e]{algorithm2e}

\SetCommentSty{mycommfont}
\SetEndCharOfAlgoLine{}

\usepackage{authblk}

\author[1]{Quentin Klopfenstein$^*$}
\author[2]{Quentin Bertrand$^*$}
\author[2]{Alexandre Gramfort}
\author[3]{Joseph Salmon}
\author[4]{Samuel Vaiter}
\affil[1]{Universit\'e de Bourgogne, Institut de Mathématiques de Bourgogne, Dijon, France}
\affil[2]{Universit\'e Paris-Saclay, Inria, CEA, Palaiseau, France}
\affil[3]{IMAG, Universit\'e de Montpellier, CNRS, Montpellier, France}
\affil[4]{CNRS and IMB, Universit\'e de Bourgogne, Dijon, France}
\affil[$^*$]{Equal contribution}

\title{Model identification and local linear convergence of coordinate descent}

\usepackage{xcolor}
\definecolor{linkcolor}{RGB}{83,83,182}
\definecolor{citecolor}{RGB}{128,0,128}
\hypersetup{
    colorlinks=true,
    citecolor=citecolor,
    linkcolor=linkcolor,
    urlcolor=linkcolor
}

\usepackage[capitalise,nameinlink]{cleveref} 
\usepackage{booktabs}  
\usepackage{quentin_shortcuts}

\begin{document}
\maketitle
\begin{abstract}
  \input{sections/0_abstract}
\end{abstract}

\setcounter{tocdepth}{2}
\input{sections/1_intro}
\input{sections/2_model_id}
\input{sections/3_local_conv_rates}
\input{sections/4_experiments}
\input{sections/5_conclusion}
\section*{Acknowledgements}
This work was partly supported by ANR GraVa ANR-18-CE40-0005, by Appel à projet Plan cancer 18CP134-00 de l’INSERM/CNRS-IMB and ERC Starting Grant SLAB ERC-StG-676943.

\clearpage
\bibliographystyle{plainnat}
\bibliography{suppid}
\clearpage
\appendix
\input{sections/9_1_model_id.tex}
\input{sections/9_2_local_conv_rates}
\end{document}

%% file: sections/0_abstract.tex
For composite nonsmooth optimization problems, Forward-Backward algorithm achieves model identification (\eg support identification for the Lasso) after a finite number of iterations, provided the objective function is regular enough.
Results concerning coordinate descent are scarcer and model identification has only been shown for specific estimators, the support-vector machine for instance.
In this work, we show that cyclic coordinate descent achieves model identification in finite time for a wide class of functions.
In addition, we prove explicit local linear convergence rates for coordinate descent.
Extensive experiments on various estimators and on real datasets demonstrate that these rates match well empirical results.

%% file: sections/1_intro.tex
\section{Introduction}
%
\subsection{Coordinate descent}
Over the last two decades, coordinate descent (CD) algorithms have become a powerful tool to solve large scale optimization problems \citep{Friedman_Hastie_Hofling_Tibshirani07,Friedman_Hastie_Tibshirani10}.
Many applications coming from machine learning or compressed sensing have lead to optimization problems that can be solved efficiently via CD algorithms: the Lasso \citep{Tibshirani96,Chen_Donoho_Saunders98}, the elastic net \citep{Zou_Hastie05} or support-vector machine \citep{Boser_Guyon_Vapnik92}.
All the previously cited estimators are based on an optimization problem which can be written:
\begin{align}\label{general_opt}
    x^\star \in
    \argmin_{
        x \in \mathbb{R}^{p}}
        \{ \Phi(x) \eqdef f(x) + \underbrace{\sum_{j=1}^p g_j(x_j)}_{\eqdef g(x)
        }\} \enspace ,
\end{align}
with $f$ a convex smooth (\ie with a Lipschitz gradient) function and $g_j$ proper closed and convex functions.
In the past twenty years, the popularity of CD algorithms has greatly increased due to the well suited structure of the new optimization problems mentioned above (\ie separability of the nonsmooth term), as well as the possible parallelization of the algorithms \citep{Fercoq_Richtarik15}.

The key idea behind CD (\Cref{alg:bcd}) is to solve small and simple subproblems iteratively until convergence.
More formally, for a function $\Phi:\mathbb{R}^p \mapsto \mathbb{R}$, the idea is to minimize successively one dimensional functions $\Phi_{|x_j}: \mathbb{R} \mapsto \mathbb{R}$, updating only one coordinate at a time, while the others remain unchanged.
There exists many variants of CD algorithms, the main branching being:
\begin{itemize}
    \item \textbf{The index selection.}
    There are different ways to choose the index of the updated coordinate at each iteration.
    The main variants can be divided in three categories, \textbf{cyclic} CD \citep{Tseng_Yun09} when the indices are chosen in the set $[p]\eqdef\{1, \dots, p\}$ cyclically.
    \textbf{Random} CD \citep{Nesterov12}, where the indices are chosen following a given random distribution.
    Finally, \textbf{greedy} CD \citep{Nutini_Schmidt_Laradji_Friedlander_Koepke15}
    \textit{picks} an index, optimizing a given criterion: largest decrease of the objective function, or largest gradient norm (Gauss-Southwell rule), for instance.
    \item \textbf{The update rule.} There also exists several possible schemes for the coordinate update: exact minimization, coordinate gradient descent or prox-linear update (see \citealt[Sec. 2.2]{Shi_Tu_Xu_Yin16} for details).
\end{itemize}
\textbf{In this work, we will focus on \textbf{cyclic} CD with \textbf{prox-linear} update rule (\Cref{alg:bcd})}: a popular instance, \eg the one coded in popular packages such as \texttt{glmnet} \citep{Friedman_Hastie_Hofling_Tibshirani07} or \texttt{sklearn} \citep{Pedregosa_etal11}.

Among the methods of coordinate selection, \textbf{random} CD has been extensively studied, especially by \citet{Nesterov12} for the minimization of a smooth function $f$.
It was the first paper proving global non-asymptotic $1/k$ convergence rate in the case of a smooth and convex $f$.
This work was later extended to composite optimization $f + \sum_j g_j$ for nonsmooth separable functions \citep{Richtarik_Takac2014,Fercoq_Richtarik15}.
Refined convergence rates were also shown by
\citet{Shalev_Tewari2011,Shalev_Zhang2013}.
These convergence results have then been extended to coordinate descent with equality constraints \citep{Necoara_Patrascu2014} that induce non-separability
as found in the SVM dual problem in the presence of the bias term.
Different distributions have been considered for the index selection such as uniform distribution \citep{Fercoq_Richtarik15, Nesterov12, Shalev_Tewari2011, Shalev_Zhang2013}, importance sampling \citep{Leventhal_Lewis2010, Zhang2004} and arbitrary sampling \citep{Necoara_Patrascu2014, Qu_Richtarik2016,Qu_Richtarik2016b}.

On the opposite, theory on \textbf{cyclic} coordinate descent is more fuzzy,
the analysis in the cyclic case being more difficult.
First, \citet{Luo_Tseng1992,Tseng01,Tseng_Yun09,Razaviyayn_Hong_Luo2013} have shown convergence results for (block) CD algorithms for nonsmooth optimization problems (without rates\footnote{Note that some local rates are shown in \citet{Tseng_Yun09} but under some strong hypothesis.}).
Then, \citet{Beck_Tetruashvili13} showed $1/k$ convergence rates for Lipschitz convex functions and linear convergence rates in the strongly convex case.
\citet{Saha_Tewari13} proved $1/k$ convergence rates for composite optimization $f + \normin{.}_1$ under "isotonicity" condition.
\citet{Sun_Hong2015,Hong_Wang2017} have extended the latter results and showed $1/k$ convergence rates with improved constants for composite optimization $f + \sum_j g_j$.
\citet{Li_Zhao_Arora_Liu_Hong2017} have extended the work of \citet{Beck_Tetruashvili13} to the nonsmooth case and refined their convergence rates in the smooth case.
Finally, as far as we know, the work by \citet{Xu_Yin2017} is the first one tackling the problem of local linear convergence. They have proved local linear convergence under the very general Kurdyka-Lojasiewicz hypothesis, relaxing convexity assumptions.
Following the line of work by \citet{Liang_Fadili_Peyere14}, we use a more restrictive framework (see \Cref{sub:pb_statement}) that allows to achieve finer results: model identification as well as improved local convergence results.
%
\subsection{Model identification}
\label{sub:manifold_id}
%
%
Nonsmooth optimization problems coming from machine learning such as the Lasso or the support-vector machine (SVM) generally generate solutions lying onto a low-complexity model (see \Cref{def:gsupp} for details).
For the Lasso, for example, a solution $x^\star$ has typically only a few non-zeros coefficients: it lies on the model set
$T_{x^\star} = \{u\in\mathbb{R}^p: \text{supp}(u)\subseteq\text{supp}(x^\star)\}$,
where $\text{supp}(x)$ is the support of $x$, \ie the set of indices corresponding to the non-zero coefficients.
A question of interest in the literature is: does the algorithm achieve model identification after a finite number of iterations?
Formally, does it exist $K>0$ such that for all $k>K$, $x^{(k)} \in T_{x^\star}$? For the Lasso the question boils down to ``does it exist $K>0$ such that for all $k>K$,
$\supp(x^{(k)}) \subseteq \supp(x^\star)$''?
This finite time identification property is paramount for features selection \citep{Tibshirani96}, but also for potential acceleration methods \citep{Massias_Gramfort_Salmon18} of the CD algorithm, as well as model calibration \citep{Bertrand_Klopfenstein_Blondel_Vaiter_Gramfort_Salmon2020}.

Finite model identification was first proved in \citet{Bertsekas76} for the projected gradient method with non-negative constraints.
In this case, after a finite number of steps the sparsity pattern of the iterates is the same as the sparsity pattern of the solution. It means that for $k$ large enough, $x^{(k)}_i = 0$ for all $i$ such that $x^\star_i=0$.
Then, many other results of finite model identification have been shown in different settings and for various algorithms.
For the projected gradient descent algorithm, identification was proved for polyhedral constraints \citep{Burke_More88}, for general convex constraints \citep{Wright1993}, and even non-convex constraints \citep{Hare_Lewis04}.
More recently, identification was proved for proximal gradient algorithm \citep{MercierVijayasundaram1979,Combettes_Wajs05}, for the $\ell_1$ regularized problem \citep{Hare2011}.
\citet{Liang_Fadili_Peyere14,Liang_Fadili_Peyre17,Vaiter_Peyre_Fadili17} have shown model identification and local linear convergence for proximal gradient descent.
These results have then been extended to other popular machine learning algorithms such as SAGA, SVRG \citep{Poon_Liang_Schonlieb2018} and ADMM \citep{Poon_Liang2019}.
To our knowledge, CD has not been extensively studied with a similar generality.
Some identification results have been shown for CD, but only on specific models \citep{She_Scmidt2017,Massias_Vaiter_Gramfort_Salmon19} or variants of CD \citep{Wright2012}, in general, under restrictive hypothesis.
 The authors are not aware of generic model identification results for CD \Cref{alg:bcd}.
%
\subsection{Notation}
%
\paragraph{General notation.}
We write $\normin{\cdot}$ the Euclidean norm on vectors.
For, $x, \gamma \in \bbR^p$, the weighted norm is denoted $\normin{x}_\gamma \eqdef \sqrt{\sum_{j=1}^p \gamma_j x_j^2}$.
For a differentiable function
$\psi : \mathbb{R}^p \mapsto \mathbb{R}^p$,
at $x\in\mathbb{R}^p$,
we write $\cJ\psi (x) \in \mathbb{R}^{p \times p}$ the Jacobian of $\psi$ at $x$.
For a set $S$, we denote by $S^c$ its complement.
We denote $[p]=\{1,\dots,p\}$.
Let $(e_j)_{j=1}^{p}$ be the vectors of the canonical base of $\mathbb{R}^p$.
We denote the coordinatewise multiplication of two vectors $u$ and $v$ by $u \odot v$ and by $u \odot M$ the row wise multiplication between a vector and a matrix.
We denote by $\mathcal{B}(x, \epsilon)$ the ball of center $x$ and radius $\epsilon$.
The spectral radius of a matrix $M$ is denoted $\rho(M)$.
\paragraph{Convex analysis.}
We recall the definition of the proximity operator of a convex function $g$,
    for any $\gamma > 0$:
    \begin{align*}
        \prox_{\gamma g}(x)
        =
        \argmin_{y\in \mathbb{R}^{p}} \frac{1}{2\gamma}
        || x - y ||^{2} + g(y)
        \enspace .
    \end{align*}
Let $\mathcal{C}\subset \mathbb{R}^{p}$ be a convex set, $\text{aff}(\mathcal{C})$ denotes its affine hull, the smallest affine set containing $\mathcal{C}$, and $\text{ri}(\mathcal{C})$ denotes its relative interior (the interior of its affine hull).
The indicator function of $\cC$ is the function defined for any $x\in \bbR^p$ by
\begin{align}
    \delta_{\cC}(x)=
    \begin{cases}
        0 \text{ if }x \in \cC \\
        +\infty \text{ otherwise } \enspace .
    \end{cases}
\end{align}
The domain of a function $f$ is defined as $\text{dom}(f) = \{x\in\mathbb{R}^p: f(x)<+\infty\}$.
For a convex function $f$, $\partial f(x)$ denotes its subdifferential at $x$  and is given by $\partial f(x) = \{s\in \mathbb{R}^{p}: f(y) \geq f(x) + \langle s, y-x\rangle, \forall y\in \text{dom}(f) \}$.
We denote by $L_j$ the coordinatewise Lipschitz constants of $\nabla_j f$, \ie, for all $x \in \bbR^p$, $h_j \in \bbR$:
\begin{align}
    ||\nabla_j f(x + e_j h_j) - \nabla_j f(x)||
    \leq
    L_j |h_j|
    \enspace .
\end{align}
\paragraph{Coordinate descent.}
We denote $0 < \gamma_j \leq 1 / L_j$ the local step size and $\gamma = (\gamma_1,\dots,\gamma_p)^\top $.
To prove model identification we need to ``keep track'' of the iterates: following the notation from \cite{Beck_Tetruashvili13} coordinate descent can be written:
{\fontsize{4}{4}\selectfont
\begin{algorithm}[H]
\SetKwInOut{Input}{input}
\SetKwInOut{Init}{init}
\SetKwInOut{Parameter}{param}
\caption{\textsc{Proximal coordinate descent}}
\Input{$\gamma_1, \dots, \gamma_p \in \mathbb{R}_+, n_{\text{iter}} \in \bbN, x^{(0)} \in \bbR^p
$}
    \For(\tcp*[f]{index selection}){$k = 0,\dots, n_{\text{iter}}$
    }{
        $x^{(0, k)} \! \leftarrow \! x^{(k)}$

        \For{$j = 1, \hdots, p$}{
            $x^{(j, k)} \! \leftarrow \! x^{(j-1, k)}$

            $x_j^{(j, k)} \! \leftarrow \!
        \prox_{\gamma_j g_j} \!
           \left( x_j^{(j-1, k)} - \gamma_j \nabla_j f (x^{(j-1, k)}) \right) $
        }
        $x^{(k+1)} \! \leftarrow \! x^{(p, k)}$
    }
\Return{
    $x^{n_{\text{iter}}+1}$
    }
\label{alg:bcd}
\end{algorithm}
}
\subsection{Assumptions on composite problem} \label{sub:pb_statement}
We consider the optimization problem defined in \Cref{general_opt} with the following assumptions:
\begin{assumption}[Smoothness]\label{ass:conv_div_lip}
    $f$ is a convex and differentiable function, with a Lipschitz gradient.
   \end{assumption}
   \begin{assumption}[Proper, closed, convex]\label{ass:closed_proper_conv}
    For any $j \in [p], g_j$ is proper, closed and convex.
   \end{assumption}
   \begin{assumption}[Existence]\label{ass:non_empty}
    The problem admits at least one solution:
       \begin{align}
           \argmin_{x\in \bbR^p} \Phi(x) \neq \emptyset
           \enspace  .
       \end{align}
   \end{assumption}
   \begin{assumption}[Non degeneracy]\label{ass:non_degeneracy}
    The problem is non-degenerate: for any $x^{\star} \in \argmin_{x\in \bbR^p} \Phi(x)$
    \begin{equation}
        - \nabla f (x^\star) \in \text{ri}\left(\partial g(x^\star)\right)
        \enspace .
    \end{equation}
    \Cref{ass:non_degeneracy} can be seen as a generalization of qualification constraints \citep[Sec. 1]{Hare_Lewis2007}.
   \end{assumption}
%
\subsection{Contributions}
%
With mild assumptions on the $g_j$ functions, for the \textbf{cyclic} proximal coordinate descent algorithm:
\begin{itemize}
    \item We prove finite time model identification (\cref{thm:finite_identification}).
    \item We provide local linear convergence rates (\cref{thm:local_linear}).
    \item We illustrate our results on multiple real datasets and estimators (\Cref{sec:expes}) showing that our theoritical rates match the empirical ones.
\end{itemize}
%

%% file: sections/2_model_id.tex

\section{Model identification for CD}
\label{sec:identification}
As stated before, the solutions of the Lasso
are structured.
Using an iterative algorithm like \Cref{alg:bcd} to find an approximate solution (since we stop after a finite number of iterations) brings the question of structure recovery.
For the Lasso, the underlying structure, also called model \citep{candes_recht12}, is identified by the Forward-Backward algorithm.
It means that after a finite number of iterations, the iterative algorithm leads to an approximate solution that shares a similar structure than the true solution of the optimization problem \citep{Liang_Fadili_Peyere14, Vaiter_Peyre_Fadili17,Fadili_Malick_Peyre2018}.
For the Lasso, the underlying model is related to the notion of support: \ie the non-zero coefficients for the Lasso, and it can be generalized for the case of completely separable functions as follows:
\begin{definition}[{Generalized support, \citealt{Sun_Jeong_Nutini_Schmidt2019}}]\label{def:gsupp}
    We call \emph{generalized support} $\cS_x \subseteq [p]$ the set of indices $j \in [p]$ where $g_j$ is differentiable at $x_j$:
    \begin{align}
        \cS_x
        \eqdef \{j \in [p]: \partial g_j(x_j) \text{ is a singleton}\}
        \enspace .
    \end{align}
\end{definition}
This notion can be unified with the definition of model subspace from \citet[Sec. 3.1]{Vaiter_Golbabaee_Fadili_Peyre2015}:
\begin{definition}[{Model subspace, \citealt{Vaiter_Golbabaee_Fadili_Peyre2015}}]
    We denote the model subspace at $x$:
    \begin{align}
        T_x = \{u\in \mathbb{R}^{p}: \forall j \in \mathcal{S}_x^{c}, u_j = 0\} \enspace .
    \end{align}
\end{definition}
See \Cref{lemma:local_c2_PS} in \Cref{app:proofs_model_id} for details.

\textbf{Examples in machine learning.} \\
\textit{The $\ell_1$ norm.} The function $g(x) = \sum_{i=1}^{p}|x_i|$ is certainly the most popular nonsmooth convex regularizer promoting sparsity.
Indeed, the $\ell_1$ norm generates structured solution with model subspace~\citep{Vaiter_Peyre_Fadili17}.
We have that $\mathcal{S}_x = \{j \in [p] : x_j \neq 0\}$ since $|\cdot|$ is differentiable everywhere but not at $0$, and the model subspace reads:
\begin{align}
    T_x
    =
    \{u\in\mathbb{R}^p: \text{supp}(u)\subseteq\text{supp}(x)\}
    \enspace .
\end{align}
\textit{The box constraints indicator function $\delta_{[0, C]}$.}
    This indicator function appears for instance in box constrained optimization problems such as the dual problem of the SVM.
    Let $\mathcal{I}_x^0 = \left\{j\in [p]: x_j=0 \right\}$ and $\mathcal{I}_x^C = \left\{j\in [p]: x_j=C \right\}$, then
    \begin{align}
         T_x & = \{u\in\mathbb{R}^p: \mathcal{I}_x^0 \subseteq \mathcal{I}_u^0 \text{ and } \mathcal{I}_x^C \subseteq \mathcal{I}_u^0\} \nonumber .
    \end{align}
    For the SVM, model identification boils down to finding the active set of the box constrained quadratic optimization problem after a finite number of iterations.

We now turn to our identification result.
To ensure model identification, we need the following (mild) assumption:
\begin{assumption}[Locally $\mathcal{C}^2$]\label{ass:locally_c2}
    For all $j \in \mathcal{S}_{x^\star}$, $g_j$ is locally $\mathcal{C}^2$ around $x_j^\star$, and $f$ is locally $\mathcal{C}^2$ around $x^\star$.
\end{assumption}
It is satisfied for the Lasso and the dual SVM problem mentioned above, but also for sparse logistic regression or elastic net.
The following theorem shows that the CD (\Cref{alg:bcd}) has the model identification property with local constant step size $0 < \gamma_j \leq 1 / L_j$:
\begin{theorem}[Model identification of CD]
    \label{thm:finite_identification}
    Consider a solution $x^\star\in \argmin_{x\in\bbR^p} \Phi(x)$ and $\mathcal{S} = \mathcal{S}_{x^\star}$. Suppose
    \begin{enumerate}
        \item \Cref{ass:conv_div_lip,ass:closed_proper_conv,ass:non_empty,ass:non_degeneracy,ass:locally_c2} hold.
        \item The sequence $(x^{(k)})_{k\geq 0}$ generated by \Cref{alg:bcd} converges to $x^\star$.
    \end{enumerate}
    Then, \Cref{alg:bcd} identifies the model after a finite number of iterations,  which means that there exists $K>0$ such that for all
    $k\geq K$,
    $x_{\cS^c}^{(k)} = x_{\cS^c}^{\star}$.
\end{theorem}
This result implies that for $k$ large enough, $x^{(k)}$ shares the support of $x^\star$ (potentially smaller).

\begin{proofsketch}[{\Cref{thm:finite_identification}}]
    \begin{itemize}
        \item First we show that \Cref{ass:conv_div_lip,ass:closed_proper_conv,ass:non_empty,ass:non_degeneracy,ass:locally_c2} implies that $g$ is \emph{partly smooth} \citep{Lewis2002} at $x^\star$ relative to the affine space $x^\star + T_{x^\star}$.
        \item Then we show that for the CD \Cref{alg:bcd}:
        $\text{dist}\left ( \partial \Phi(x^{(k)}), 0 \right ) \rightarrow 0$, when $ k \rightarrow \infty$, enabling us to apply \cite{Hare_Lewis04}[Thm. 5.3].
    \end{itemize}
A full proof of \Cref{thm:finite_identification} can be found in \Cref{app:proofs_model_id}.
The first point is shown in appendix \Cref{lemma:local_c2_PS}.
We show the second point below:

\begin{myproof}
As written in \Cref{alg:bcd}, one update of coordinate descent reads:
\begin{align}
    \frac{1}{\gamma_j} x_j^{(j-1, k)}
    & - \nabla_j f
    \left (
        x^{(j-1, k)}
    \right ) \nonumber
    - \frac{1}{\gamma_j} x_j^{(j, k)}
    \in \partial g_j \left  ( x_j^{(j, k)} \right)
    \\
    \frac{1}{\gamma_j} x_j^{(k)}
    & - \nabla_j f
    \left (
        x^{(j-1, k)}
    \right ) \nonumber
    - \frac{1}{\gamma_j} x_j^{(k+1)}
    \in \partial g_j \left  ( x_j^{(k+1)} \right)
    \enspace .
\end{align}
Since $g$ is separable with non-empty subdifferential, the coordinate wise subdifferential of $g$ is equal to the subdifferential of $g$, we then have
\begin{align}
   \frac{1}{\gamma} \odot x^{(k)} -
   &\left (
        \nabla_j f \left ( x^{(j-1, k)} \right )
    \right )_{j \in [p]}
    - \frac{1}{\gamma} \odot x^{(k+1)}
    \nonumber
    \\
    &\in \partial g(x^{(k+1)}) \enspace ,
\end{align}
which leads to
\begin{align}
    \frac{1}{\gamma} \odot x^{(k)} &
    - \left (
        \nabla_j f \left ( x^{(j-1, k)} \right )
    \right )_{j \in [p]}
    - \frac{1}{\gamma} \odot x^{(k+1)} \nonumber \\
    &  + \nabla f(x^{(k+1)})\in \partial \Phi(x^{(k+1)})
    \label{eq:subdiff_inclusion}
    \enspace .
\end{align}
To prove support identification using \citet[Theorem 5.3]{Hare_Lewis04}, we need to bound the distance between $\partial \Phi(x^{(k+1)})$ and $0$, using \Cref{eq:subdiff_inclusion}:
\begin{align}
    & \text{dist}\left ( \partial \Phi(x^{(k+1)}), 0 \right )^2
    \nonumber \\
    & \leq \sum_{j=1}^p
    \left | \frac{x_j^{(k)}}{\gamma_j} -
        \nabla_j f ( x^{(j-1, k)} )
        - \frac{x_j^{(k+1)}}{\gamma_j} + \nabla_j f(x^{(k+1)})
     \right |^2 \nonumber \\
    & \leq ||x^{(k)} - x^{(k+1)}||_{{\gamma^{-1}}}^2 \nonumber \\
    & + \sum_{j=1}^{p}
        \left |
            \nabla_j f \left ( x^{(j-1, k)}\right )
            - \nabla_j f \left (x^{(k+1)} \right )
        \right |^2
    \nonumber \\
    & \leq ||x^{(k)} - x^{(k+1)}||_{{\gamma^{-1}}}^2
     + L^2 \sum_{j=1}^{p} \normin{ x^{(j-1, k)} - x^{(k+1)} }^2
     \nonumber \\
    &\leq
    \underbrace{
     ||x^{(k)} - x^{(k+1)}||_{{\gamma^{-1}}}^2
    + L^2 \sum_{j=1}^p
        \sum_{j' \geq j}^p
            \left | x_{j'}^{(k)} - x_{j'}^{(k+1)} \right |^2
    }_{
        \rightarrow 0 \text{ when } k \rightarrow \infty}
    \nonumber
    \enspace .
\end{align}

\end{myproof}
We thus have:
\begin{itemize}
    \item $\text{dist}\left ( \partial \Phi(x^{(k+1)}), 0 \right ) \rightarrow 0$
    \item  $\Phi(x^{(k)})\rightarrow \Phi(x^\star)$ because $\Phi$ is prox-regular (since it is convex, see \citealt{Poliquin_Rockafellar1996b}) and subdifferentially continuous.
\end{itemize}
Then the conditions to apply \citet[Th. 5.3]{Hare_Lewis04}
are met and hence we have model identification after a finite number of iterations.
\end{proofsketch}

\textbf{Comments on \Cref{thm:finite_identification}.}
It unifies several results found in the literature: \cite{Massias_Vaiter_Gramfort_Salmon19} showed model identification for the Lasso, solved with coordinate descent, but requiring uniqueness assumption.
\cite{Nutini_Laradji_Schmidt2017} showed some identification results under strong convexity assumption on $f$.
\Cref{thm:finite_identification} do not rely on any uniqueness, strong convexity, or local strong convexity hypothesis.
Even if the solution of the optimization problem defined in \Cref{general_opt} is not unique, CD achieves model identification.

%% file: sections/3_local_conv_rates.tex

\section{Local convergence rates}
%
In this section, we prove local linear convergence of the CD \Cref{alg:bcd}. After model identification, there exists a regime where the convergence towards a solution of \Cref{general_opt} is linear.
Local linear convergence was already proved in various settings such as for ISTA and FISTA algorithms (\ie with an $\ell_1$ penalty, \citealt{Tao2016local})  and then for the general Forward-Backward algorithm \citep{Liang_Fadili_Peyere14}.

Local linear convergence requires an additional assumption: \textit{restricted injectivity}.
It is classical for this type of analysis as it can be found in \citet{Liang_Fadili_Peyre17} and \citet{ Poon_Liang2019}.
\begin{assumption} (Restricted injectivity)\label{ass:restricted_injectivity}
    For a solution $x^\star\in \argmin_{x\in\bbR^p} \Phi(x)$,
    the restricted Hessian to its generalized support $\mathcal{S} = \mathcal{S}_{x^\star}$ is definite positive, \ie
    \begin{align}
        \nabla^{2}_{\mathcal{S}, \mathcal{S}}f(x^\star) \succ 0 \enspace .
    \end{align}
\end{assumption}
For the Lasso, \Cref{ass:restricted_injectivity} is a classical necessary condition to ensure uniqueness of the minimizer \citep{Fuchs04}.

In order to study local linear convergence, we consider the fixed point iteration of a complete epoch (an epoch is a complete pass over all the coordinates).
A full epoch of CD can be written:
\begin{equation}\label{eq:operator_psi}
    x^{(k+1)} = \psi(x^{(k)}) \eqdef \mathcal{P}_p \circ \hdots \circ \mathcal{P}_1 (x^{(k)})\enspace ,
\end{equation}
where $\cP_j$ are coordinatewise sequential applications of the proximity operator $\cP_j : \bbR^p \rightarrow \bbR^p$:
\begin{align}
        x
        & \mapsto
        \left (\begin{array}{c}
            x_1\\
            \vdots \\
            x_{j-1} \\
            \prox_{\gamma_j g_j}
            \big( x_j - \gamma_j\nabla_j f(x) \big)\\
            x_{j+1} \\
            \vdots \\
            x_p
        \end{array}
        \right )
        \nonumber \enspace .
\end{align}
Thanks to model identification (\Cref{thm:finite_identification}), we are able to prove that once the generalized support is correctly identified, there exists a regime where CD algorithm converges linearly towards the solution:
\begin{theorem}[Local linear convergence]\label{thm:local_linear}
    Consider a solution $x^\star\in \argmin_{x\in\bbR^p} \Phi(x)$ and $\mathcal{S} = \mathcal{S}_{x^\star}$. Suppose
    \begin{enumerate}
        \item \Cref{ass:conv_div_lip,ass:closed_proper_conv,ass:non_empty,ass:non_degeneracy,ass:locally_c2,ass:restricted_injectivity} hold.
        \item The sequence $(x^{(k)})_{k\geq 0}$ generated by \Cref{alg:bcd} converges to $x^\star$.
        \item The model has been identified \ie there exists $K\geq 0$ such as for all $k\geq K$
        \begin{align*}
            x_{\cS^c}^{(k)} = x_{\cS^c}^\star
            \enspace.
        \end{align*}
    \end{enumerate}

    Then $(x^{(k)})_{k\geq K}$ converges linearly towards $x^\star$.
    More precisely, for any $\nu \in [\rho(\cJ \psi_{\cS, \cS}(x^\star)), 1[$, there exists $K>0$ and a constant $C$ such that for all $k\geq K$,
    \begin{equation*}
        \normin{x_\cS^{(k)} - x_\cS^\star}
        \leq C \nu^{(k-K)}
        \normin{x_\cS^{(K)} - x_\cS^\star}
        \enspace .
    \end{equation*}
\end{theorem}
The complete proof of \Cref{thm:local_linear} can be found in \Cref{app:proofs_local_lin_conv}.
%
\begin{figure*}[tb]
    \centering
    \centering
    \includegraphics[width=0.7\linewidth]{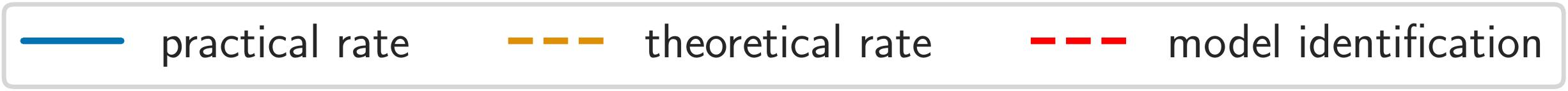}
    \centering
    \includegraphics[width=0.8\linewidth]{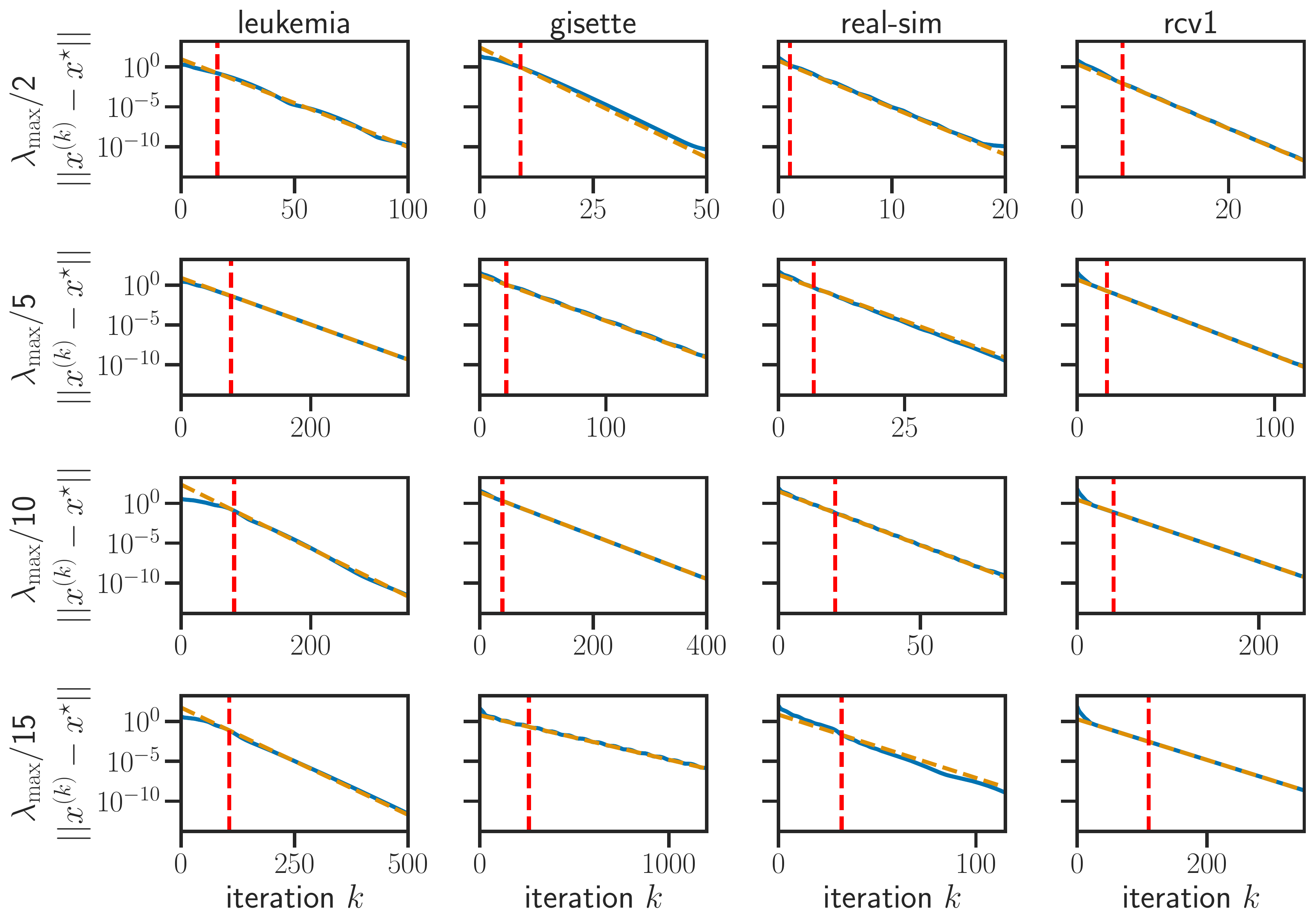}
    \caption{
        \textbf{Lasso, linear convergence.}
    Distance to optimum, $\normin{x^{(k)} - x^\star}$, as a function of the number of iterations $k$, on 4 different datasets: \emph{leukemia}, \emph{gisette}, \emph{rcv1},
    and \emph{real-sim}.}
    \label{fig:linear_convergence_lasso}
\end{figure*}

\begin{proofsketch}[\Cref{thm:local_linear}]
    \begin{itemize}
        \item A key element of the proof is to consider a full epoch of CD: it can be written as a fixed point iteration:
        $x^{(k+1)} = \psi(x^{(k)})$ (see \Cref{eq:operator_psi}).
        \item We then show that the proximal operators, $\prox_{\gamma_j g_j}$, evaluated at $x_j^\star - \gamma_j \nabla_j f(x^\star)$ are differentiable (for $j \in [p]$).
        Once stated, the differentiability of the proximal operator allows us to write the Taylor expansion of $\psi$:
        \begin{align}
            x^{(k+1)} - x^\star
            &= \psi(x^{k}) - \psi(x^\star)  \nonumber\\
            &= \langle \cJ \psi(x^\star), x^{(k)} - x^\star \rangle + o(\normin{x^{(k)} - x^\star}).
            \nonumber
        \end{align}
        \item Capitalizing on model identification (\Cref{thm:finite_identification}) we start from $x_{\mathcal{S}^c}^{(k)} = x_{\mathcal{S}^c}^\star$ and show the bound $\rho(\cJ \psi_{\cS, \cS}(x^\star))<1$ on the spectral radius of the restricted Jacobian of $\psi$ at $x^\star$:  $\cJ \psi_{\cS, \cS}(x^\star)$.
        \item Finally, all the conditions are met to apply \citet[Th. 1, Sec. 2.1.2]{Polyak1987}. The latter reference provides sufficient conditions for local linear convergence of sequences based non linear fixed point iterations.
    \end{itemize}
\end{proofsketch}

%% file: sections/4_experiments.tex
%
%
%
%
\section{Experiments}\label{sec:expes}
%
%
\begin{figure*}[tb]
    \centering
    \centering
    \includegraphics[width=0.7\linewidth]{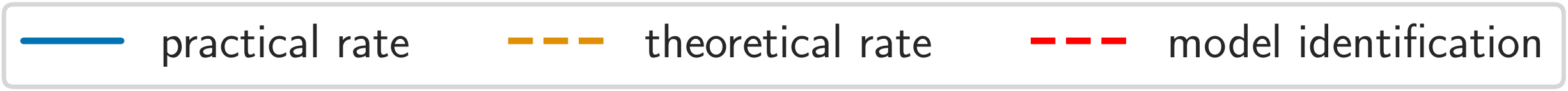}
        \centering
        \includegraphics[width=0.8\linewidth]{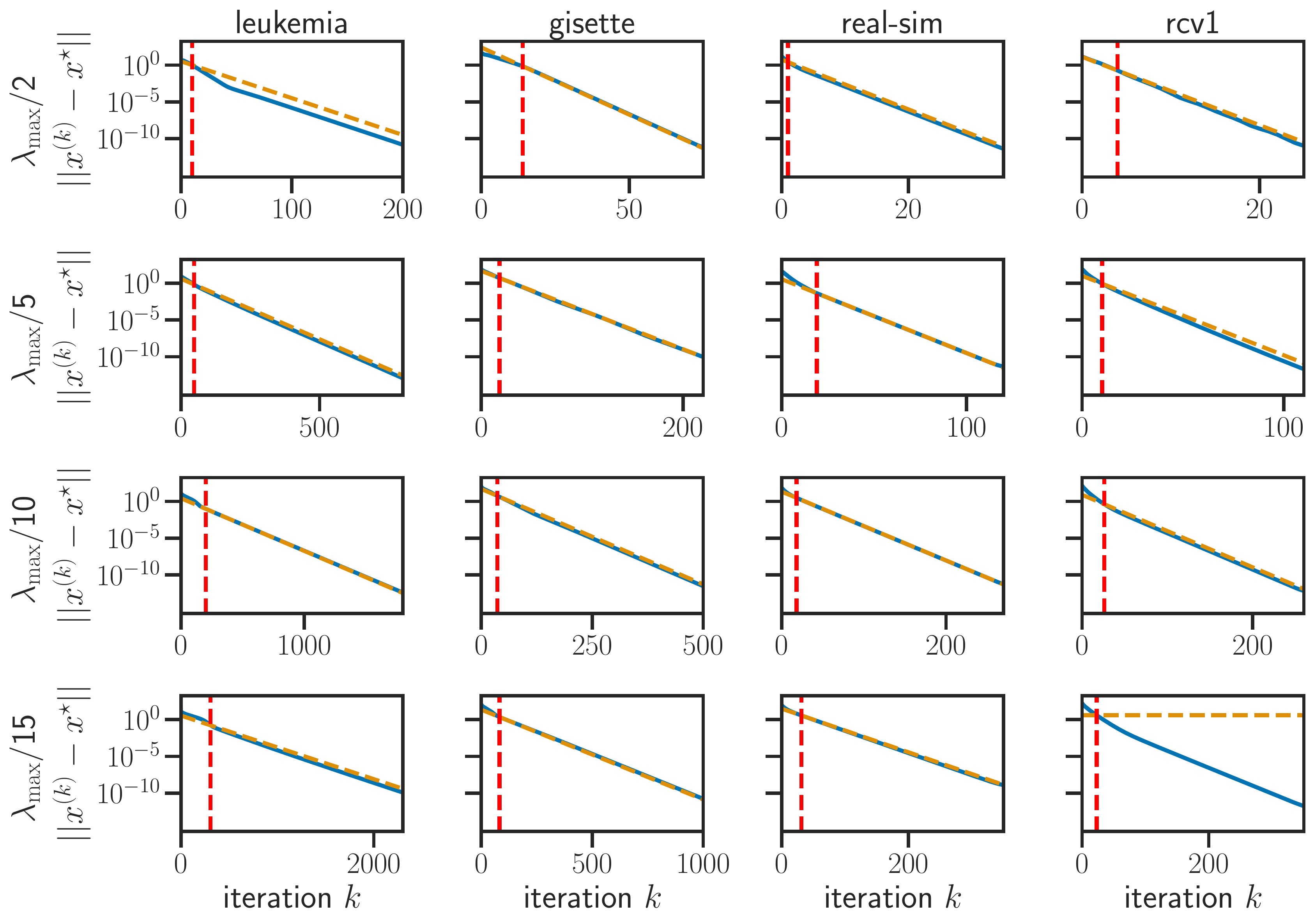}
    \caption{
        \textbf{Sparse logistic regression, linear convergence.}
    Distance to optimum, $\normin{x^{(k)} - x^\star}$, as a function of the number of iterations $k$, on 4 different datasets: \emph{leukemia}, \emph{gisette}, \emph{rcv1},
    and \emph{real-sim}.}
    \label{fig:linear_convergence_logreg}
\end{figure*}
We now illustrate \Cref{thm:finite_identification,thm:local_linear} on multiple datasets and estimators:
the Lasso, the logistic regression and the SVM.
In this section, we consider a design matrix $A\in \mathbb{R}^{n\times p}$ and a target $y\in\mathbb{R}^{n}$ for regression (Lasso) and $y\in \{-1, 1\}^{n}$ for classification (logistic regression and support-vector machine).
We used classical datasets from \texttt{libsvm} \citep{Chang_Lin11} summarized in \Cref{table:summary_data}.

In \Cref{fig:linear_convergence_lasso,fig:linear_convergence_logreg,fig:linear_convergence_svm} the distance of the iterates to the optimum, $\normin{x^{(k)} - x^\star}$ as a function of the number of iterations $k$ is plotted as a solid blue line.
The vertical red dashed line represents the iteration $k^\star$ where the model has been identified by CD (\Cref{alg:bcd}) illustrating \Cref{thm:finite_identification}.
The yellow dashed line represents the theoretical linear rate from \Cref{thm:local_linear}.
\Cref{thm:local_linear} gives the slope of the dashed yellow line, the (arbitrary) origin point of the theoretical rate line is chosen such that blue and yellow lines coincide at identification time, \ie all lines intersect at this point.
More precisely, if $k^\star$ denotes the iteration where model identification happens, the equation of the dashed yellow line is:
\begin{equation}
    h(k)
    = \normin{x^{(k^\star)} - x^\star} \times
     \rho(\cJ \psi_{\cS, \cS}(x^\star))^{(k - k^\star) }
    \enspace .
\end{equation}
Once a solution $x^\star$ has been computed, one can calculate $\cJ \psi_{\cS, \cS}(x^\star)$ and its spectral radius for each estimator.

For the experiments we used three different estimators that we detail here.
\begin{table}[tb]
    \center
    \caption{Characteristics of the datasets.}
    \begin{tabular}{ccccc}
        \hline
        Datasets & \#samples $n$ & \#features $p$ & density \\
        \hline
        leukemia & \num{38} & \num{7129} & \num{1} \\
        gisette & \num{6000} & \num{4955} & \num{1} \\
        rcv1 & \num{20242} & \num{19959} & \num{3.6e-3} \\
        real-sim & \num{72309} & \num{20958} & \num{2.4e-3} \\
        20news & \num{5184} & \num{155148} & \num{1.9e-3} \\
        \hline
    \end{tabular}
    \label{table:summary_data}
\end{table}

\textbf{Lasso.} \citep{Tibshirani96}
The most famous estimator based on a nonsmooth optimization problem may be the Lasso. For a design matrix $A \in \bbR^{n \times p}$ and a target $y \in \bbR^n$ it writes:
\begin{align}\label{eq:lasso}
    \argmin_{x \in \bbR^p}
    \frac{1}{2n}|| Ax - y||^{2}
    + \lambda ||x||_1
    \enspace .
\end{align}
The CD update for the Lasso is given by
\begin{align}
    x_j \leftarrow \text{ST}_{\gamma_j\lambda}\left(x_j - \gamma_j A_{:,j}^{\top}(y - Ax)\right) \enspace ,
\end{align}
where $\text{ST}_{\lambda}(x)= \text{sign}(x) \cdot \max(|x| - \lambda, 0)$.
The solution of \Cref{eq:lasso} is obtained using \Cref{alg:bcd} with constant stepsizes
$1  / \gamma_j =  \frac{||A_{:, j}||^{2}}{n}$.
\textbf{Sparse logistic regression.}
The sparse logistic regression is an estimator for classification tasks.
It is the solution of the following optimization problem, for a design matrix $A \in \bbR^{n \times p}$ and a target variable $y \in \{-1, 1 \}^n$, with $\sigma(z) \eqdef \frac{1}{1 + e^{-z}}$:
\begin{equation}\label{eq:sparse_logreg}
    \argmin_{x \in \bbR^p}
    - \frac{1}{n} \sum_{i=1}^{n} \log \sigma(y_i x^{\top} A_{i,:})
    + \lambda \norm{x}_1
    \enspace .
\end{equation}

The CD update for the sparse logistic regression is
\begin{align}
    x_j \leftarrow \text{ST}_{\gamma_j\lambda}\left(x_j - \gamma_j A_{:,j}^{\top} (y \odot (\sigma(y \odot Ax) - 1))\right) \enspace .
\end{align}
The constant stepsizes for the CD algorithm to solve \Cref{eq:sparse_logreg} are given by $1  / \gamma_j = \frac{||A_{:, j}||^{2}}{4n}$.
%

\begin{figure*}[tb]
    \centering
    \centering
    \includegraphics[width=0.7\linewidth]{linear_convergence_logreg_legend}
    \centering
    \includegraphics[width=0.95\linewidth]{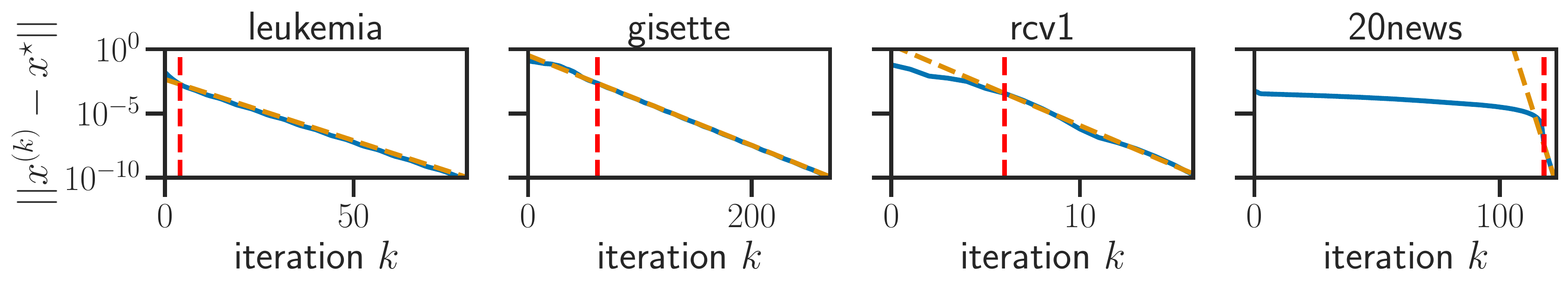}
    \caption{
        \textbf{Support vector machine, linear convergence.}
    Distance to optimum, $\normin{x^{(k)} - x^\star}$, as a function of the number of iterations $k$, on 4 different datasets: \emph{leukemia}, \emph{gisette}, \emph{rcv1} and \emph{20news}.}
    \label{fig:linear_convergence_svm}
\end{figure*}
\textbf{Support-vector machine.} \citep{Boser_Guyon_Vapnik92}
The support-vector machine (SVM) primal optimization problem is, for a design matrix $A \in \bbR^{n \times p}$ and a target variable $y \in \{-1, 1 \}^n$:
\begin{align}
    \argmin_{x \in \bbR^p}
    \frac{1}{2} \normin{x}^2
    + C  \sum_{i=1}^{n}\max \left(1 - y_i x^{\top} A_{i,:}, 0 \right)
    \enspace .
\end{align}
The SVM can be solved using the following dual optimization problem:
\begin{align}\label{eq:svm}
    \argmin_{w \in \bbR^n} &
    \frac{1}{2} w^\top (y \odot A) (y \odot A)^\top w
    - \sum_{i=1}^n w_i \nonumber \\
    & \text{subject to } 0 \leq w_i \leq C
    \enspace .
\end{align}
The CD update for the SVM reads:
\begin{align}
    w_i \leftarrow \mathcal{P}_{[0,C]}\left(w_i - \gamma_i((y\odot A)_{i,:}^{\top}(y\odot Aw) - 1\right)) \enspace ,
\end{align}
where $\mathcal{P}_{[0,C]}(x) = \min(\max(0, x), C)$.
The stepsizes of the CD algorithm to solve \Cref{eq:svm} are given by $1 / \gamma_i = ||(y \odot A)_{i,:}||^2$.
The values of the regularization parameter $C$ for each dataset from \Cref{fig:linear_convergence_svm} are given in \Cref{table:CvalueSVM}.
\begin{table}[tb]
    \center
    \caption{$C$ values for SVM.}
    \begin{tabular}{c|ccccc}
        dataset & leukemia & gisette & rcv1 & 20news \\
        \hline
        C value & $10$ & $1.5$ $10^{-2}$ & $1.5$ $10^{-2}$ & $5$ $10^{-1}$ \\
    \end{tabular}
    \label{table:CvalueSVM}
\end{table}

\textbf{Comments on \Cref{fig:linear_convergence_lasso,fig:linear_convergence_logreg,fig:linear_convergence_svm}.}
Finite time model identification and local linear convergence are illustrated on the Lasso, the sparse logistic regression and the SVM in \Cref{fig:linear_convergence_lasso,fig:linear_convergence_logreg,fig:linear_convergence_svm}.
As predicted by \Cref{thm:finite_identification}, the relative model is identified after a finite number of iterations.
For the Lasso (\Cref{fig:linear_convergence_lasso}) and the sparse logistic regression (\Cref{fig:linear_convergence_logreg}), we observe that as the regularization parameter gets smaller, the number of iterations needed by the CD algorithm to identify the model increases.
To our knowledge, this is a classical empirical observation, that is not backed up by theoretical results.
After identification, the convergence towards a solution is linear as predicted by \Cref{thm:local_linear}.
The theoretical local speed of convergence provided by \Cref{thm:local_linear} seems like a sharp estimation of the true speed of convergence as illustrated by the three figures.

Note that on \Cref{fig:linear_convergence_lasso,fig:linear_convergence_logreg,fig:linear_convergence_svm} high values of $\lambda$ (or small values of $C$) were required for the restricted injectivity \Cref{ass:restricted_injectivity} to hold.
Indeed, despite its lack of theoretical foundation, it is empirically observed that, in general, the larger the value of $\lambda$, the smaller the cardinal of the generalized support: $|\cS|$.
It makes the restricted injectivity \Cref{ass:restricted_injectivity}: $\nabla_{\cS, \cS}^2 f(x^\star) \succ 0$ easier to be satisfied.
For instance, for $\lambda = \lambda_{\max} / 20$, the restricted injectivity \Cref{ass:restricted_injectivity} was not verified for a lot of datasets for the Lasso and the sparse logistic regression (\Cref{fig:linear_convergence_lasso,fig:linear_convergence_logreg}).
In the same vein, values of $C$ for the SVM had to be chosen small enough, in order to make $|\cS|$ not too large (\Cref{fig:linear_convergence_svm}).

Note that finite time model identification is crucial to ensure local linear convergence, see for instance \emph{20news} dataset on \Cref{fig:linear_convergence_svm}.
However there exists very few quantitative theoretical results for the convergence speed of the model identification.
\citet{Nutini_Schmidt_Hare2019,Sun_Jeong_Nutini_Schmidt2019} tried to obtain some rates on the identification, quantifying ``how much the problem is qualified'', \ie how much \Cref{ass:non_degeneracy} is satisfied.
But these theoretical results do not seem to explain fully the experimental results of the CD: in particular the identification speed of the model compared to other algorithms.

\textbf{Limits.}
We would like to point out the limit of our analysis illustrated for the case of $\lambda = \lambda_{\text{max}} / 15$ for the sparse logistic regression and the \textit{rcv1} dataset in \Cref{fig:linear_convergence_logreg}.
In this case, the solution may no longer be unique.
The support gets larger and \Cref{ass:restricted_injectivity} is no longer met.
In this case, the largest eigenvalue of $\cJ\psi_{\cS, \cS}(x^\star)$ is exactly one, which leads to the constant rate observed in \Cref{fig:linear_convergence_logreg}.
Despite the largest eigenvalue being exactly $1$, a regime of locally linear convergence toward a (potentially non unique) minimizer is still observed.
Linear convergence of non-strongly convex functions starts to be more and more understood \citep{Necoara_Nesterov_Glineur2019}.
\Cref{fig:linear_convergence_logreg} with $\lambda = \lambda_{\text{max}} / 15$ for \textit{rcv1} suggests extensions of \citet{Necoara_Nesterov_Glineur2019} could be possible in the nonsmooth case.
%

%% file: sections/5_conclusion.tex
%
%
%
%

\textbf{Conclusion and future work.}
In conclusion, we show finite time model identification for coordinate descent \Cref{alg:bcd} (\Cref{thm:finite_identification}).
Thanks to this identification property we were able to show local linear rates of convergence (\Cref{thm:local_linear}).
These two theoretical results were illustrated on popular estimators (Lasso, sparse logistic regression and SVM dual) and popular machine learning datasets (\Cref{sec:expes}).

A first natural extension of this paper would be to investigate block coordinate minimization: \Cref{thm:finite_identification} could be extended for blocks under general partial smoothness assumption \citep{Hare_Lewis04}.
However, it seems that \Cref{thm:local_linear} would require a more careful analysis.
A second extension could be to show linear convergence without the restricted injectivity (\Cref{ass:restricted_injectivity}), paving the way for a generalization of \citet{Necoara_Nesterov_Glineur2019} as suggested by \Cref{fig:linear_convergence_logreg}.

%% file: sections/9_1_model_id.tex
%
\section{Proofs of model identification (\Cref{thm:finite_identification})}\label{app:proofs_model_id}
%
%
Model identification often relies on the assumption that the nonsmooth function $g$ is regular enough, or more precisely \textit{partly smooth}.
Loosely speaking, a partial smooth function behaves smoothly as it lies on the related model and sharply if we move normal to that model.
Formally, we recall the definition of partly smooth functions restricted to the case of proper, lower semicontinuous and convex functions.
\begin{definition}[Partial smoothness]\label{def:partly_smooth}
    Let $g: \mathbb{R}^{p}\mapsto \mathbb{R}$ be a proper closed convex function. $g$ is said to be partly smooth at $x$ relative to a set $\mathcal{M} \subseteq \mathbb{R}^{n}$ if there exists a neighbourhood $\mathcal{U}$ of $x$ such that
    \begin{itemize}
        \item \textbf{(Smoothness)} $\mathcal{M} \cap \mathcal{U}$ is a $\mathcal{C}^{2}$-manifold and $g$ restricted to $\mathcal{M} \cap \mathcal{U}$ is $\mathcal{C}^{2}$,
        \item \textbf{(Sharpness)} The tangent space of $\mathcal{M}$ at $x$ is the model tangent space $T_x$ where $T_x = \text{Lin}(\partial g(x))^{\perp}$,
        \item \textbf{(Continuity)} The set valued mapping $\partial g$ is continuous at $x$ relative to $\mathcal{M}$.
    \end{itemize}
\end{definition}
The class of partly smooth functions was first defined in \citet{Lewis2002}.
It encompasses a large number of known nonsmooth machine learning optimization penalties, such as the $\ell_1$-norm or box constraints to only name a few, see \citet[Section 2.1]{Vaiter_Peyre_Fadili17} for details.
Interestingly, this framework enables powerful theoretical tools on model identification such as \citet{Hare_Lewis04}[Thm. 5.3].
For separable functions, next lemma gives an explicit link between the generalized support \Cref{def:gsupp} \citep{Sun_Jeong_Nutini_Schmidt2019}  and the framework of partial smooth functions \citep{Hare_Lewis04}.
\begin{lemma}\label{lemma:local_c2_PS}
    Let $x^{\star} \in \dom(g)$.
    If for every $j \in \mathcal{S}_{x^{\star}}$, $g_j$ is locally $C^2$ around $x_j^\star$ (\Cref{ass:locally_c2}), then $g$ is partly smooth at $x^\star$ relative to $x^\star + T_{ x^\star}$.
\end{lemma}
\begin{proof}
    We need to prove the three properties of the partial smoothness (\Cref{def:partly_smooth}).

    \textbf{Smoothness.}
    Let us write $\mathcal{M}_{x^\star} = x^\star + T_{x^\star}$ the affine space directed by the model subspace and pointed by $x^\star$. In particular, it is a $C^2$-manifold.

    For every $j \in \gsupp_{x^{\star}}$, $g_j$ is locally $C^2$ around $x^\star_j$, hence there exists a neighborhood $U_j$ of $x^\star_j$ such that the restriction of $f$ to $U$ is twice continuously differentiable. For $j \in \gsupp_{x^\star}^c$, let's write $U_j = \mathbb{R}$.
    Take $U = \bigotimes_{j \in [p]} U_j$. This a neighborhood of $x^\star$ (it is open, and contains $x^\star$).
    Consider the restriction $g_{|\mathcal{M}_{x^\star}}$ of $g$ to $\mathcal{M}_{x^\star}$.
    It is $C^2$ at each point of $U$ since each coordinates (for $j \in \gsupp_{x^\star}$) are $C^2$ around $U_j$.

    \textbf{Sharpness.}
    Since $g$ is completly separable, we have that
    \begin{math}
        \partial g(x^\star) =
        \partial g_1(x^\star_1) \times \hdots \times \partial g_p(x^\star_p)
    \end{math}.
    Note that $\partial g_j(x^\star_j)$ is a set valued mapping which is equal to the singleton $\{\nabla_j g(x^\star_j)\}$ if $g_j$ is differentiable at $x^\star_j$ or it is equal to an interval.
    The model tangent space $T_{x^\star}$ of $g$ at $x^\star$ is given by
    \begin{align}
        T_{x^\star} = \text{span}(\partial g(x^\star))^{\perp}
        \quad \text{where} \quad
        \text{span}(\partial g(x^\star)) = \text{aff}(\partial g(x^\star)) - e_{x^\star}
        \enspace ,
    \end{align}
    with
    \begin{align}
        e_{x^\star} = \argmin_{e\in \text{aff}(\partial g(x^{\star}))} ||e|| \enspace,
    \end{align}
    called the model vector.

    In the particular case of separable functions, we have that
    \begin{align*}
        \text{aff}\left (\partial g(x^\star)\right )& = \text{aff}\left (\partial g_1(x^\star_1) \times \hdots \times \partial g_p(x^\star_p) \right ) = \text{aff}\left( \partial g_1(x^\star_1)\right) \times \hdots \times \text{aff}\left( \partial g_p(x^\star_p)\right)
        \enspace .
        \nonumber
    \end{align*}
    In this case,
    \begin{align}
        \text{aff}\left( \partial g_j(x_j^\star)\right) =
        \begin{cases}
            \{\nabla_j g(x^\star_j)\} & \text{ if } j \in \mathcal{S}_{x^{\star}}\\
            \mathbb{R} & \text{ otherwise}
        \end{cases}
\quad \text{and} \quad
        e_{x^\star_j} =
        \begin{cases}
            \nabla_j g(x^\star_j) & \text{ if } j \in \mathcal{S}_{x^{\star}}\\
            0 & \text{ otherwise} \enspace .
        \end{cases}
    \end{align}
    Thus we have that
    \begin{align*}
        \text{span}\left(\partial g(x^\star)\right)
        &= \text{aff}\left( \partial g(x^\star)\right) - e_{x^\star} 
        = \{ x\in \mathbb{R}^{p}: \forall j' \in \cS_{x^\star}, x_{j'} = 0\} \enspace .
    \end{align*}
    Then
    \begin{align}
        T_{x^\star} = \text{span}\left( \partial g(x^\star)\right)^{\perp}
        =
        \{x\in \mathbb{R}^{p}: \forall j' \in \mathcal{S}_{x^\star}^c, x_{j'} = 0\}
        \enspace .
    \end{align}

    \textbf{Continuity.}
    We are going to prove that $\partial g$ is inner semicontinuous at $x^\star$ relative to $\mathcal{M}_{x^\star}$, \ie that for any sequence $(x^{(k)})$ of elements of $\mathcal{M}_{x^\star}$ converging to $x^\star$ and any $\bar{\eta} \in \partial g(x^\star)$, there exists a sequence of subgradients $\eta^{(k)} \in \partial g(x^{(k)})$ converging to $\bar{\eta}$.

    Let $x^{(k)}$ be a sequence of elements of $\mathcal{M}_{x^\star}$ converging to $x^\star$, or equivalently, let $t^{(k)}$ be a sequence of elements of $T_{x^\star}$ converging to $0$, and let $\bar \eta \in \partial g(x^\star)$.

    For $j \in \gsupp_{x^\star}$,
    we choose $\eta_j^{(k)} \eqdef g_j'(x^\star_j + t^{(k)}_j)$,
    using the smoothness property we have $\eta_j^{(k)} \eqdef \bar \eta_j$.
    For all $j \in  \gsupp_{x^\star}^c$ $x_j^{(k)} = x_j^\star$ we choose
    $\eta_j^{(k)} \eqdef \bar \eta_j$,
    since $x^{(k)} \in \cM_{x^\star}$, we have $\eta_j^{(k)}\in \partial g(x^{(k)})$.

    We have that $\eta^{(k)} \in \partial g(x^k)$ and $\eta^{(k)}$ converges towards $\bar \eta$ since $g_j'$ is $C^1$ around $x^\star_j$ for $j \in \gsupp_{x^\star}$, hence, $g_j'(x^\star_j + t^{(k)}_j)$ converges to $g_j'(x^\star_j) = \bar \eta_j$.
    Thus, it proves that $g$ is partly smooth at $x^{\star}$ relative to $x^{\star} + T_{x^{\star}}$.
\end{proof}
The end of the proof of \Cref{thm:finite_identification} is contained in~\Cref{sec:identification}.

%% file: sections/9_2_local_conv_rates.tex
%
\section{Proofs of local linear convergence (\Cref{thm:local_linear})}\label{app:proofs_local_lin_conv}
To simplify the notations in this section, $\mathcal{S} \eqdef \mathcal{S}_{x^\star}$.
Let us also write the element of $\mathcal{S}$ as follows:
\begin{math}
    \mathcal{S} = \{j_1, \hdots, j_{|\cS|}\}
\end{math}.
The first point of this proof is to write the CD algorithm as a fixed point iteration.
A full epoch of CD can be written as
\begin{equation}
x^{(k+1)} = \psi(x^{(k)}) = \mathcal{P}_p \circ \hdots \circ \mathcal{P}_1 (x^{(k)})\enspace .
\end{equation}
We also define $\pi^{x_{\cS^c}^\star} : \bbR^{|\cS|} \to \bbR^p$ for all $x_{\cS} \in \bbR^{|\cS|}$ and all $j \in \cS$ by
\begin{align}
    \left( \pi^{x_{\cS^c}^\star}(x_{\cS}) \right)_j =
    \begin{cases}
        x_j       & \text{if } j \in \cS \\
        x_j^\star & \text{if } j \in \cS^c \enspace,
    \end{cases}
\end{align}
and for all $j_s \in \cS$, $\tilde{\cP}_{j_s}^{x_{\cS^c}^\star} : \bbR^{|\cS|} \to \bbR^{|\cS|}$ is the function defined for all $x_{\cS} \in \bbR^{|\cS|}$ and all $j \in \cS$ by
\begin{align}
    \left( \tilde{\cP}_{j_s}^{x_{\cS^c}^\star} (x_{\cS}) \right)_j
    =
    \begin{cases}
        x_j &
        \text{if } j \neq j_s \\
        \prox_{\gamma_j g_j}
        \left( x_{j_s} - \gamma_{j_s}\nabla_{j_s} f(\pi^{x_{\cS^c}^\star}(x_{\cS})) \right) &
        \text{if } j = j_s \enspace.
    \end{cases}
\end{align}
%
Once the model is identified (\Cref{thm:finite_identification}), we have that there exists $K\geq 0$ such that for all $k\geq K$, we have that
\begin{align}
    x^{(k)}_{\mathcal{S}^c} = x^{\star}_{\mathcal{S}^{c}}
    \quad \text{and} \quad
    x^{(k+1)}_\mathcal{S}
    = \tilde{\psi} ( x^{(k)}_\mathcal{S})
    & \eqdef \mathcal{P}_{j_{|\cS|}}^{x^\star_{\mathcal{S}^c} } \circ \hdots \circ \mathcal{P}_{j_1}^{x^\star_{\mathcal{S}^c} }(x^{(k)}_\mathcal{S})
    \enspace .
\end{align}
When no confusion is possible, we denote by $\tilde{\mathcal{P}_j}$ the function $\tilde{\cP}_{j_s}^{x_{\cS^c}^\star}$, hence still dependant on $x_{\cS^c}^\star$.
The following lemma shows that $\tilde{\mathcal{P}_j}$ is differentiable at the optimum.
\begin{lemma}\label{lemma:diff_prox}
    For all $j\in \mathcal{S}$, $\tilde{\mathcal{P}_j}$ is differentiable at $x^{\star}_{\mathcal{S}}$.
\end{lemma}
\begin{proof}

    From \Cref{ass:locally_c2}, we know there exists a neighboorhood of $x^{\star}_j$ denoted $\mathcal{U}$ such that, for $j\in \mathcal{S}$, the restriction of $g_j$ to $\mathcal{U}$ is $\mathcal{C}^{2}$ on $\mathcal{U}$.
    In particular, it means that $x_j^\star$ is a differentiable point of $g_j$ and given a pair $(u,v) \in \mathcal{U} \times \bbR^p$ such that
    \begin{equation} \label{eq:id_prox}
        u = \prox_{\gamma_jg_j}(v) \in \mathcal{U}
        \enspace ,
    \end{equation}
    we have $\frac{1}{\gamma_j}(v-u)\in \partial g_j(u)$ becomes
    \begin{align}
        \frac{1}{\gamma_j}(v - u) & = g_j'(u) \Leftrightarrow
        v  = u + \gamma_j g_j'(u) \Leftrightarrow
        v  = (\text{Id} + g_j')(u) \enspace .
    \end{align}
    Let $H(u) = (\text{Id} + g_j')(u)$, since $g_j$ is twice differentiable at $u$, we have that
    \begin{align}
        H'(u) = 1 + \gamma_j g_j''(u) \enspace .
    \end{align}
    Thus, $H': \mathcal{U} \mapsto \mathbb{R}$ is continuous and then $H: \mathcal{U} \mapsto \mathbb{R}$ is continuously differentiable.
    Hence $F(v, u) \eqdef v - H(u)$ is $\mathcal{C}^{1}$ and $F(v, u) = 0$.
    By convexity of $g$, we have $g_j''(u) \geq 0$ and
    \begin{align}
        \frac{\partial F}{\partial u}(v, u)
        =
        - H'(u)
        =
        - 1 - \gamma_j g''(u) \neq 0
        \enspace.
    \end{align}
    Using the implicit functions theorem, we have that there exists an open interval $\mathcal{V}\subseteq \mathbb{R}$ with $v \in \mathcal{V}$ and a function $h: \mathcal{V} \mapsto \mathbb{R}$ which is $\mathcal{C}^{1}$ such as $u = h(v)$.

    Using \eqref{eq:id_prox} we thus have with the choice $u=x_j^\star$, $v=x_j^\star - \gamma_j \nabla_j f(x^\star)$ that the map $h$ coincides with $\prox_{\gamma_jg_j}$ on $\mathcal{V}$ and
    is differentiable at $v=x_j^\star - \gamma_j \nabla_j f(x^\star) \in \cV$.
    It follows that $\tilde{\cP}_{j}$ is differentiable at $x_{\cS}^\star$.
\end{proof}
For the sake of completness, we show that in fact $\prox_{\gamma_j g_j}$ is also differentiable on the complement of the generalized support at $x^\star_j-\nabla_j f(x^\star)$.
\begin{lemma}
    For all $j \in \cS^c$, $\prox_{\gamma_j g_j}$ is constant around $x^\star_j-\nabla_j f(x^\star)$. Moreover, the map $x \mapsto \prox_{\gamma_j g_j}(x_j-\nabla_j f(x))$ is differentiable at $x^\star$ with gradient $0$.
\end{lemma}
\begin{proof}
    Let	$\partial g_j (x^\star_j)= [a;b]$ and
    let $z_j^\star = x^\star_j-\nabla_j f(x^\star)$, then combining the fixed point equation and \Cref{ass:non_degeneracy} leads to:
    \begin{align}
        \frac{1}{\gamma_j}(z_j^\star -x_j^\star)\in \text{ri}\left (\partial g_j(x^\star_j)\right )=]a;b[ \enspace .
    \end{align}
    Thus,
    \begin{align}
        z_j^\star \in ]\gamma_j a + x^\star_j; \gamma_j b + x_j^\star[ \enspace .
    \end{align}
    For all $v\in ]\gamma_j a + x^\star_j; \gamma_j b + x_j^\star[ $, we have
    \begin{math}
        \frac{1}{\gamma_j}(v - x_j^\star)\in ]a;b[= \text{ri}\left (\partial g_j(x^\star_j)\right )
    \end{math},
    \ie $\prox_{\gamma_j g_j}(v) = x^\star_j$.
    As $f$ is $\mathcal{C}^2$ in $x^\star$, we have that $x\mapsto \prox_{\gamma_j g_j}(x_j-\nabla_j f(x))$ is differentiable at $x^\star$ with gradient being $0$.
\end{proof}

From \Cref{lemma:diff_prox}, we have that $\tilde{\mathcal{P}_j}$ is differentiable at $x^{\star}_\mathcal{S}$ for all $j\in \mathcal{S}$.
Since $x^{\star}$ is an optimal point, the following fixed points equation holds:
\begin{align}
    x^{\star}_j = \prox_{\gamma_j g_j}\left(x^{\star}_j - \gamma_j \nabla_j f(x^{\star})\right) \enspace .
\end{align}
The map $\tilde{\psi}$ is then differentiable at $x^{\star}_\mathcal{S}$ since it is obtained as the composition of differentiable functions and that each function $\tilde{\mathcal{P}_j}$ is evaluated at a differentiable point (only one coordinate change at each step).

To compute, the Jacobian of $\tilde{\mathcal{P}_j}$ at $x^{\star}_\mathcal{S}$, let us first notice that

\begin{align}
    \cJ\mathcal{P}_j(x^{\star}_\mathcal{S})^\top =
    \left (
        \begin{array}{c|c|c|c|c|c|c}
            e_1 & \hdots & e_{j-1} & v_j & e_{j+1} & \hdots & e_s
        \end{array}
    \right) \enspace ,
\end{align}
where $v_j = \partial_x \prox_{\gamma_j g_j}\left(z_j^{\star}\right)\left(e_j - \gamma_j \nabla^{2}_{j,:}f(x^{\star})\right)$ and
$z_j^{\star} = x^{\star}_j - \gamma_j \nabla_j f(x^{\star})$.
This matrix can be rewritten as
\begin{align}
    \cJ \tilde{\mathcal{P}_j}(x^{\star}_\mathcal{S})
    &= \Id_{|S|} - e_je_j^{\top} + \partial_x \prox_{\gamma_j g_j}
        \left(z_j^{\star}\right)
        \left (
            e_j e_j^\top - \gamma_j e_j e_j^\top \nabla^{2}f(x^{\star})
        \right )
        \nonumber \\
    &= \Id_{|S|} - e_je_j^{\top} \gamma_j \partial_x \prox_{\gamma_j g_j}
        \left(z_j^{\star}\right)
        \left (
            \diag(u) + \nabla^{2}f(x^{\star})
        \right)
        \nonumber \\
    &= \Id_{|S|} - e_je_j^{\top} \gamma_j \partial_x \prox_{\gamma_j g_j}
        \left(z_j^{\star}\right) M
        \nonumber \\
    &= M^{-1/2}
    \left(\Id_{|S|} - M^{1/2} e_j e_j^{\top} \gamma_j \partial_x \prox_{\gamma_j g_j}\left(z_j^{\star}\right) M^{1/2}\right)
    M^{1/2}
    \nonumber \\
    & = M^{-1/2}
        \left(
            \Id_{|S|} - B_j
        \right)M^{1/2}
        \enspace ,
\end{align}
where
\begin{equation}\label{eq:def-mat-M}
    M \eqdef
    \nabla_{\cS, \cS}^{2} f(x^\star)
    + \diag
    \left ( u \right ) \enspace,
\end{equation}
and $ u \in \bbR^{|\cS|} $ is defined for all $j \in \cS$ by
\begin{equation}
    u_j =
    \begin{cases}
        \frac{1}{\gamma_j\partial_{x}\prox_{\gamma_j g_j}(z_j^\star)} - \frac{1}{\gamma_j}
        & \text{if } \prox_{\gamma_j g_j}(z_j^\star) \neq 0 \\
        0 & \text{otherwise,}
    \end{cases}
\end{equation}
and
\begin{equation}\label{eq:def-mat-Bj}
    B_j = M^{1/2}_{:,j}\gamma_j \partial_x \prox_{\gamma_j g_j}\left(z_j^{\star}\right) M_{:,j}^{1/2\top} .
\end{equation}

Since only one coordinate change at each step, the chain rule leads to
\begin{align}
    \cJ\tilde{\psi}(x^{\star}_{\mathcal{S}}) & = \cJ\mathcal{P}_{j_s}(x^{\star}_{\mathcal{S}})\cJ\mathcal{P}_{j_{s-1}}(x^{\star}_{\mathcal{S}})\hdots \cJ\mathcal{P}_{j_1}(x^{\star}_{\mathcal{S}}) \nonumber \\
                 & =  M^{-1/2}\underbrace{\left(\text{Id} - B_{j_s}\right) \hdots (\text{Id} - B_{j_1})}_{A} M^{1/2} \nonumber
\end{align}
The next series of lemma will be usefull to prove that the spectral radius $\rho\left(\cJ\tilde{\psi}(x^{\star}_\mathcal{S})\right)<1$.
\begin{lemma}\label{lemma:M_sym_def_pos}
    The matrix $M$ defined in~\eqref{eq:def-mat-M} is symmetric definite positive.
\end{lemma}
\begin{proof}
Using the non-expansivity of the prox,
 and the property $\partial_{x}\prox_{\gamma_j g_j}(z_j^\star) > 0$ for $j \in \cS$,
$\diag(u)$ is a symmetric semidefinite matrix,
so $M$ is a sum of a symmetric definite positive matrix and a symmetric semidefinite matrix, hence $M$ is symmetric definite positive.
\end{proof}

\begin{lemma}
   For all $j\in \mathcal{S}$, the matrix $B_j$ defined in~\eqref{eq:def-mat-Bj} has spectral norm bounded by 1, \ie $\norm{B_j}_2 \leq 1$.
\end{lemma}
\begin{proof}
    $B_j$ is a rank one matrix which is the product of
    $\gamma_j\partial_{x}\prox_{\gamma_jg_j}(z^{\star}_{j}) M_{:, j}^{1/2}$ and $M_{:,j}^{1/2\top}$,
    its non-zeros eigenvalue is thus given by
    \begin{align}
        \normin{B_j}_2
        &=
        \left |
        M^{1/2\top}_{:,j}  \gamma_j\partial_{x}\prox_{\gamma_j g_j}(z_j^{\star})
         M_{:, j}^{1/2}
         \right |
        \nonumber \\
        &=
        \left |
        \gamma_j\partial_{x}\prox_{\gamma_j g_j}(z^{\star}_{j}) M_{j, j}
        \right |
        \nonumber \\
        & =
        \left |
        \gamma_j\partial_{x}\prox_{\gamma_j g_j}(z^{\star}_{j})
        \left (
            \underbrace{\nabla_{j, j}^{2} f(x^{\star})}_{0 \leq }
                +
                \underbrace{\left (
                \frac{1}{\gamma_j\partial_{x}\prox_{\gamma_j g_j}(z^{\star}_{j})}
                - \frac{1}{\gamma_j}
                \right )}_{0\leq }
        \right )
        \right | \enspace.
        \end{align}
        By positivity of the two terms,
        \begin{align}
            \normin{B_j}_2
            & =
            \gamma_j\partial_{x}\prox_{\gamma_j g_j}(z^{\star}_{j})
                \underbrace{
                    \nabla_{j, j}^{2} f(x^{\star})}_{\leq L_{j}\leq \frac{1}{\gamma_j}}
                    + \left (
                    1 - \partial_{x}\prox_{\gamma_j g_j}(z^{\star}_{j})
                \right ) \nonumber \\
            & \leq
                \partial_{x}\prox_{\gamma_j g_j}(z^{\star}_{j})
                + \left (
                1 - \partial_{x}\prox_{\gamma_j g_j}(z^{\star}_{j})
                \right ) \nonumber\\
            & \leq 1 \enspace.
            \end{align}
\end{proof}
\begin{lemma}
    For all $j\in \mathcal{S}$, $B_j / \normin{B_j}$ is an orthogonal projector onto $\Span(M_{:, j}^{1/2})$.
\end{lemma}
\begin{proof}
    It is clear that $B_{j} / ||B_{j}||$ is symmetric. We now prove that it is idempotent, \ie $(B_{s} / ||B_{s}||)^{2} = B_{s} / ||B_{s}||.$
    \begin{align}
        B_{j}^{2} / ||B_{j}||^{2} & = (\gamma_j\partial_{x}\prox_{\gamma_j g_j}(z^{\star}_{j}))^{2} M_{:, j}^{1/2}  M_{:,j}^{1/2 \top} M_{:, j}^{1/2}  M_{:,j}^{1/2 \top} / ||B_{j}||^{2} \nonumber \\
        & = (\gamma_j\partial_{x}\prox_{\gamma_j g_j}(z^{\star}_{j}))||B_{j}|| M_{:, j}^{1/2} M_{:,j}^{1/2 \top} / ||B_{j}||^{2} \nonumber \\
        & = B_{j} / ||B_j|| \nonumber \enspace .
    \end{align}
    Hence, $B_j / ||B_j||$ is an orthogonal projector.
\end{proof}
\begin{lemma}\label{lemma:ortho_xj}
    For all $j\in \mathcal{S}$  and for all $x \in \bbR^S$, if $\normin{(\Id - B_j)x} = \norm{x}$ then $x \in \Span(M_{:,j}^{1/2})^\perp$.
\end{lemma}
\begin{proof}
    \begin{align}
        \Id - B_j
        &=
        \Id - \normin{B_j} \frac{B_j}{\normin{B_j}} \nonumber\\
        &=
        (1 - \normin{B_j}) \Id
        + \normin{B_j}_2 \Id
        - \normin{B_j}_2 \frac{B_j}{\normin{B_j}_2} \nonumber \\
        &=
        (1 - \normin{B_j}) \Id
        + \normin{B_j}
        \underbrace{\left ( \Id - \frac{B_j}{\normin{B_j}_2} \right )}_{\text{projection onto } M_{:,j}^{1/2\perp}} \enspace . \label{eq:almost_proj}
    \end{align}
    Let $x \notin \Span( M_{:,j}^{1/2})^\perp$,
    then there exists $\kappa \neq 0$,
    $x_{ M_{:,j}^{1/2\perp}} \in \Span( M_{:,j}^{1/2})^\perp$  such that
    \begin{align} \label{eq:decomposition_x}
        x = \kappa M_{:,j} + x_{ M_{:,j}^{1/2\perp}} \enspace .
    \end{align}
    Combining \Cref{eq:almost_proj,eq:decomposition_x} leads to:
    \begin{align}
        (\Id - B_j) x
        &=
        (1 - \normin{B_j}_2) x
        + \normin{B_j}_2 x_{ M_{:,j}^{1/2\perp}}
        \nonumber\\
        \normin{(\Id - B_j) x}
        &\leq
        \underbrace{|1 - \normin{B_j}_2|}_{=1 - \normin{B_j}_2} \normin{x}
        + \normin{B_j}_2 \underbrace{\normin{x_{ M_{:,j}^{1/2\perp}}]}}_{< \normin{x}} \nonumber\\
        & < \norm{x} \nonumber \enspace .
    \end{align}
\end{proof}

\begin{lemma}\label{lemma:A_1}
    The spectral norm of $A$ is bounded by 1, \ie $\normin{(\Id - B_{j_s}) \dots (\Id - B_{j_1})}_2 = ||A||_2 < 1$.
\end{lemma}
\begin{proof}
    Let $x \in \bbR^s$ such that $\normin{(\Id - B_{j_s}) \dots (\Id - B_{j_1})x} = \norm{x}$.
    Since
    \begin{align}
        \normin{(\Id - B_{j_s} \dots (\Id - B_{j_1})}_2
        \leq
        \underbrace{\normin{(\Id - B_{j_s})}_2}_{\leq 1}
        \times \dots \times
        \underbrace{\normin{(\Id - B_{j_1})}_2}_{\leq 1}
        \nonumber \enspace ,
    \end{align}
    we thus have for all $j \in \mathcal{S}$,
    \begin{math}
        \normin{(\Id - B_j)x} = \norm{x}
    \end{math}.
    One can thus successively apply \Cref{lemma:ortho_xj} which leads to:
    \begin{align}
        x &\in \bigcap_{j\in \mathcal{S}} \Span{ M_{:,j}^{1/2}}^\perp \Leftrightarrow
        x \in \Span \left ( M_{:,j_1}^{1/2}, \dots, M_{:,j_s}^{1/2} \right )^\perp \nonumber \enspace .
    \end{align}
    Moreover $M^{1/2}$ has full rank (see \Cref{lemma:M_sym_def_pos}), thus $x=0$ and
    \begin{align}
        \normin{(\Id - B_{j_s}) \dots (\Id - B_{j_1})}_2 < 1 \nonumber \enspace .
    \end{align}
\end{proof}

From \Cref{lemma:A_1}, $||A||_2<1$. Moreover $A$ and
$\cJ\tilde{\psi}(x^\star_{\mathcal{S}})$ are similar matrices, then
$\rho(\cJ\tilde{\psi}(x^\star_{\mathcal{S}})) = \rho(A) \leq ||A||_2 < 1$.

To summarize, $x^\star_\mathcal{S}$ is the solution of a fixed point equation $\tilde{\psi}(x^\star_\mathcal{S}, x^\star_{\mathcal{S}^c}) = x^\star_\mathcal{S}$.
From \Cref{lemma:diff_prox}, $\tilde{\psi}(., x^\star_{\mathcal{S}^c})$ is differentiable at $x^\star_\mathcal{S}$ and the Jacobian at $x^{\star}_\mathcal{S}$ satifies the condition $\rho(\cJ\tilde{\psi}(x^\star_\mathcal{S}))<1$.
Then all conditions are met to apply \cite{Polyak1987}[Theorem 1, Section 2.1.2] which proves local linear convergence.